\documentclass[conference]{IEEEtran}
\IEEEoverridecommandlockouts
\usepackage{cite}
\usepackage{amsmath,amssymb,amsfonts, amsthm}
\usepackage{algorithmic}
\usepackage{graphicx}
\usepackage{textcomp}
\usepackage{xcolor}
\usepackage{subcaption}

\graphicspath{{images/}}
\newtheorem{theorem}{Theorem}
\def\BibTeX{{\rm B\kern-.05em{\sc i\kern-.025em b}\kern-.08em
    T\kern-.1667em\lower.7ex\hbox{E}\kern-.125emX}}

\begin{document}

\title{Path-Weighted Integrated Gradients for Interpretable Dementia Classification\\
\thanks{Identify applicable funding agency here. If none, delete this.}
}

\author{\IEEEauthorblockN{ Firuz Kamalov}
\IEEEauthorblockA{\textit{Department of Electrical Engineering} \\
\textit{Canadian University Dubai}\\
Dubai, UAE \\
firuz@cud.ac.ae}
\and
\IEEEauthorblockN{Mohmad Al Falasi}
\IEEEauthorblockA{\textit{Department of Electrical Engineering} \\
\textit{Canadian University Dubai}\\
Dubai, UAE \\
20210002271@students.cud.ac.ae}
\and
\IEEEauthorblockN{Fadi Thabtah}
\IEEEauthorblockA{\textit{Department of Informatics} \\
\textit{Abu Dhabi School of Management}\\
Abu Dhab, UAE \\
f.fayez@adsm.ac.ae}
}

\maketitle

\begin{abstract}
Integrated Gradients (IG) is a widely used attribution method in explainable artificial intelligence (XAI).  In this paper, we introduce Path-Weighted Integrated Gradients (PWIG), a generalization of IG that incorporates a customizable weighting function into the attribution integral. This modification allows for targeted emphasis along different segments of the path between a baseline and the input, enabling improved interpretability, noise mitigation, and the detection of path-dependent feature relevance. We establish its theoretical properties and illustrate its utility through experiments on a dementia classification task using the OASIS-1 MRI dataset.  Attribution maps generated by PWIG highlight clinically meaningful brain regions associated with various stages of dementia, providing users with sharp and stable explanations. The results suggest that PWIG offers a flexible and theoretically grounded approach for enhancing attribution quality in complex predictive models.
\end{abstract}

\begin{IEEEkeywords}
integrated gradients, feature attribution, dementia detection, explanable AI
\end{IEEEkeywords}

\section{Introduction}

Deep neural networks (DNNs) have achieved state-of-the-art performance across a wide range of domains, including computer vision, natural language processing, and healthcare. Despite their success, the opaque nature of these models often prevents their adoption in safety-critical and regulated environments, where interpretability is paramount. Explainable Artificial Intelligence (XAI) aims to bridge this gap by developing techniques that elucidate the internal decision-making processes of complex models. Feature attribution methods have garnered particular attention due to their ability to quantify the influence of individual input features on model predictions.

Gradient-based attribution techniques—such as DeepLIFT, Layer-wise Relevance Propagation (LRP), Deconvolutional Networks, and Guided Backpropagation—have emerged as prominent tools in the XAI landscape \cite{Sadeghi}. However, many of the approaches fail to meet key theoretical essentials such as sensitivity and implementation invariance.
To address these issues, Sundararajan et al.~\cite{sundararajan2017axiomatic} proposed the {Integrated Gradients} (IG) method, which calculates feature importance by integrating model gradients along a linear path from a user-defined baseline to the input of interest. The IG method satisfies both sensitivity and implementation invariance as well as more stringent the completeness property. Due to its theoretical grounding IG has been widely adopted for tasks involving model auditing, feature selection, and scientific discovery.

Despite its strengths, the standard IG framework has limitations. The reliance on a fixed, straight-line path and uniform integration can obscure subtle variations in model behavior along the path from baseline to input. Moreover, certain segments of this path may traverse regions of input space that are poorly represented in the training data, leading to noisy or uninformative gradients. These limitations motivate the need for more flexible and expressive attribution techniques.

In this paper, we propose \emph{Path-Weighted Integrated Gradients} (PWIG), a novel extension of the IG framework that introduces a weight function $g(\alpha)$ into the path integral. This function modulates the contribution of gradients computed at different points along the interpolation path, thereby offering enhanced control over the attribution process. PWIG provides a more fine-grained and context-aware interpretation of model predictions and enables several practical benefits:

First, it supports \emph{focused analysis} by allowing emphasis on specific regions of the interpolation path. For instance, emphasizing gradients near the input allows to capture final-stage contributions or the gradients near the baseline to investigate early feature activations. Second, PWIG improves \emph{noise reduction} by down-weighting gradients in regions where the model exhibits unstable or erratic behavior. Third, it enables the discovery of \emph{path dependencies} by highlighting how feature importance evolves along the transition from baseline to input.

While the inclusion of a weighting function introduces greater flexibility, it also entails trade-offs—most notably, the violation of the completeness axiom unless $g(\alpha) = 1$. Nevertheless, PWIG retains other desirable properties, including sensitivity, implementation invariance, and symmetry preservation. In this study, we formalize the PWIG framework, analyze its theoretical characteristics, and demonstrate its effectiveness through empirical evaluations on a dementia dataset.

In this study, we evaluate the practical utility of PWIG in the context of a dementia classification task using the Open Access Series of Imaging Studies (OASIS-1) dataset \cite{oasis2007}.  We train a convolutional neural network (CNN) to predict diagnostic categories from the MRI volumes and then use PWIG to generate voxel-level attributions. By varying the weighting function $g(\alpha)$, we demonstrate how PWIG can reveal both early-stage biomarkers near the baseline and dominant disease-related patterns near the input. Our empirical findings indicate that PWIG provides sharp, stable, and clinically meaningful explanations.

The remainder of this paper is structured as follows. In Section~\ref{sec:relatedwork}, we review related work in explainable AI and attribution methods. Section~\ref{sec:method} formally introduces the PWIG framework, detailing its mathematical formulation, theoretical properties, and practical considerations. Section~\ref{sec:experiments} presents our empirical analysis, including the dataset description, model architecture, attribution results, and interpretability evaluation. Finally, Section~\ref{sec:conclusion} concludes the paper and outlines directions for future research.

\section{Related Work}\label{sec:relatedwork}

The increasing use of deep neural networks (DNNs) in high-stakes domains such as healthcare and finance has reinforced the importance of explainable artificial intelligence (XAI). Among the many approaches to XAI, feature attribution methods are a core category that aim to identify the contributions of individual input features to a model's prediction. Integrated Gradients (IG), introduced by Sundararajan et al.\cite{sundararajan2017axiomatic}, is one of the most widely adopted attribution methods. It computes feature importance scores by integrating the gradients of the model's output along a straight-line path from a baseline input to the actual input, thus mitigating issues such as gradient saturation.

IG adheres to several desirable axioms for attribution methods, including {completeness}, {sensitivity}, and {implementation invariance}. These properties ensure that IG provides consistent and theoretically sound explanations across equivalent models. However, recent work\cite{costanza2025riemannian} has suggested that these axioms can be satisfied by a broader family of methods, opening the door for alternative approaches.

Despite its theoretical appeal, IG suffers from several limitations. First, the method can produce noisy attributions due to erratic gradients, especially when interpolation steps deviate from the data manifold\cite{zaher2024manifold, kapishnikov2021guided}. Second, IG is highly sensitive to the choice of baseline. Selecting an inappropriate or domain-incongruent baseline (e.g., zero in medical imaging) can lead to misleading interpretations\cite{erion2021improving}. Finally, IG's reliance on many gradient evaluations makes it computationally expensive, particularly for high-dimensional data or complex architectures\cite{jha2024integrated}.

Recent research has proposed numerous improvements to IG. Guided Integrated Gradients (GIG)~\cite{kapishnikov2021guided} and Blur IG~\cite{xu2020attribution} reduce noise by modifying the integration path or leveraging scale-space theory. Integrated Decision Gradients (IDG)~\cite{jha2024integrated} restricts integration to decision-critical regions, improving both interpretability and efficiency. Methods such as RiemannOpt~\cite{swain2024riemann} and IG2\cite{zhuo2023ig2} further refine the approximation of the path integral by optimizing sampling strategies.

Baseline selection has also been a central topic. Expected Gradients (EG)~\cite{erion2021improving} treats the baseline as a distribution, improving robustness and efficiency. Tangentially Aligned IG~\cite{simpson2025tangentially} and Riemannian IG\cite{costanza2025riemannian} constrain paths to align with the data manifold, enhancing the semantic quality of attributions.

IG and its variants have been applied to a wide range of domains. In computer vision, they are used for image classification and medical diagnostics\cite{yadav2025exploring}, where visual explanation quality is crucial. In natural language processing, Sequential IG\cite{enguehard2023sequential} and Integrated Directional Gradients\cite{sikdar2021integrated} have been proposed to better handle token interactions and syntax. Time series and graph learning settings have also adopted IG-based techniques\cite{wang2024improving, simpson2025tangentially}, often with domain-specific enhancements.

Compared to other attribution methods like LIME and SHAP\cite{lundberg2017unified, ribeiro2016should}, IG offers gradient-based precision and strong theoretical grounding. However, while model-agnostic methods can be applied broadly, IG requires differentiable models. Furthermore, IG's performance depends on both baseline and path quality, whereas SHAP provides a more unified view of feature interactions at the cost of computational complexity.

Integrated Gradients remains a foundational method in the XAI landscape due to its axiomatic rigor and applicability to deep models. Nonetheless, recent work has addressed its limitations through improved integration paths, robust baselines, and better approximation techniques. These advancements enhance IG's reliability and broaden its applicability across domains.

\section{Path-Weighted Integrated Gradients}\label{sec:method}

\subsection{Mathematical Formulation}

Let $F: \mathbb{R}^n \rightarrow \mathbb{R}$ denote a predictive model, and let $x \in \mathbb{R}^n$ be the input to be explained, relative to a baseline $x' \in \mathbb{R}^n$. The classical Integrated Gradients method computes the attribution for the $i$-th input feature as:

$$
\text{IG}_i(x) = (x_i - x'_i) \times \int_{\alpha=0}^{1} \frac{\partial F(x' + \alpha(x - x'))}{\partial x_i} d\alpha.
$$

We generalize this formulation by introducing a continuous, non-negative weighting function $g(\alpha)$ defined on the interval $[0, 1]$. The Path-Weighted Integrated Gradients (PWIG) attribution for the $i$-th feature is then given by:

$$
\text{PWIG}_i(x) = (x_i - x'_i) \times \int_{\alpha=0}^{1} g(\alpha) \frac{\partial F(x' + \alpha(x - x'))}{\partial x_i} d\alpha.
$$

This formulation allows for differential weighting of gradient contributions along the interpolation path. The integral can be approximated numerically via a Riemann sum:

$$
\text{PWIG}_i^{\text{approx}}(x) = (x_i - x'_i) \times \sum_{k=1}^{m} g\left(\frac{k}{m}\right) \frac{\partial F\left(x' + \frac{k}{m}(x - x')\right)}{\partial x_i} \cdot \frac{1}{m},
$$

where $m$ is the number of integration steps.

\subsection{Theoretical Properties}

The introduction of the weighting function $g(\alpha)$ influences several theoretical properties of the attribution method. Notably, PWIG preserves the Implementation Invariance property. Since PWIG depends solely on the gradients of the function $F$ along the path from baseline to input, it yields identical attributions for functionally equivalent models.

In addition, PWIG satisfies the Sensitivity(b) axiom. When a particular feature $x_k$ has no effect on the model's output, the partial derivative $\partial F / \partial x_k$ is zero at every point along the path. As a result, the integrated gradient for that feature is also zero, ensuring that PWIG correctly assigns no attribution to irrelevant features.

PWIG also maintains the Symmetry-preserving property. If the function is symmetric with respect to a pair of features, then they receive identical attributions when their values are identical in both the input and the baseline.
This follows directly from the equality of their partial derivatives along the symmetric path. Formal proofs of the Implementation Invariance, Sensitivity(b), and Symmetry-preserving properties for PWIG are provided in the Appendix.

However, the Completeness property, which requires that the sum of all attributions equals the difference in model output between input and baseline, does not hold for PWIG. Completeness is satisfied only in the special case where $g(\alpha) = 1$ uniformly over the interval $[0, 1]$. Any deviation from this constant function introduces a discrepancy between the integrated gradients and the total output difference, thereby violating Completeness.

\subsection{Interpretability Benefits}

Path-Weighted Integrated Gradients offer several interpretability advantages beyond those available through the standard IG method. One key benefit is the ability to perform targeted attribution. By modulating the weighting function $g(\alpha)$, users can prioritize gradients near the input or the baseline depending on their analytical objectives. Emphasizing gradients closer to the input allows the user to highlight final-stage contributions, while focusing on earlier points reveals how the model begins forming its prediction from a neutral starting point.

Another important benefit of PWIG is its capacity for noise mitigation. In practice, the interpolation path between a baseline and a complex input may pass through regions of the input space where the model exhibits erratic or unstable behavior. Such regions often arise from unrealistic combinations of features. By assigning lower weights to these intermediate regions via an appropriate $g(\alpha)$, PWIG helps reduce the influence of noisy gradients, leading to clearer and more stable attribution results.

Finally, PWIG is particularly valuable in applications involving temporal or contextual dynamics, such as time-series forecasting or sequential data modeling. In these settings, feature importance may evolve significantly along the path from baseline to input. The use of a non-uniform weighting function enables the discovery of these dynamic patterns, allowing PWIG to capture more nuanced, context-dependent attributions than traditional static methods.

In summary, the Path-Weighted Integrated Gradients method extends the classical IG framework to allow flexible and context-sensitive attributions. It retains several foundational properties of IG while enabling deeper insights into model behavior through its dynamic and interpretable gradient weighting mechanism.

\section{Experiments}~\label{sec:experiments}

\subsection{Dataset and Preprocessing}

For our experimental validation, we utilized the Open Access Series of Imaging Studies (OASIS-1) dataset \cite{oasis2007}. It is a widely used benchmark in dementia classification research. The dataset contains a collection of T1-weighted MRI scans of the brain, categorized into four diagnostic labels: Non-demented, Very Mild Dementia, Mild Dementia, and Moderate Dementia. For consistency and to accommodate the input requirements of our deep learning architecture, all scans were resized to $224 \times 224$ pixels and normalized using standard ImageNet statistics. 

\subsection{Model Architecture and Training}

We designed a high-performance convolutional neural network (CNN) using the PyTorch framework to perform multi-class classification on the OASIS-1 dataset. The model consists of four convolutional blocks, each comprising a convolutional layer (with filter sizes increasing progressively from 32 to 256), followed by a ReLU activation and max-pooling layer. To improve generalization and mitigate overfitting, a dropout layer with a rate of 0.5 is applied after the final convolutional block. The output is then flattened and passed through three fully connected layers to predict the four target classes. The model achieved an accuracy of 99.80\% on the validation set providing a reliable platform for testing PWIG.

\subsection{Attribution via Path-Weighted Integrated Gradients}

In our implementation of PWIG, we employed an exponential weighting function of the form $g(\alpha) = e^{c\alpha}$ with $c=1.0$, which places greater emphasis on gradients closer to the input. The attribution integral was discretized using 50 steps, and the final attribution maps were computed by aggregating and normalizing the weighted gradients.
To enhance interpretability and reduce the influence of outlier values, attribution scores were clipped at the 60th and 95th percentiles. This filtering reduces visual noise while preserving significant regions of interest. The resulting saliency maps highlight the regions in the MRI scan that most influenced the model’s classification decision. The full details of the implementation are available on GitHub.

\subsection{Qualitative Results}

Figure~\ref{fig:oasis_attributions} presents attribution maps generated using PWIG for a selection of MRI scans across all four diagnostic categories. The highlighted regions (in green) correspond to the areas that contributed most significantly to the model’s decision. Notably, we observe consistent emphasis on anatomical structures that align with clinical markers of dementia, including cortical and subcortical regions. For instance, in the moderate and mild dementia examples, the model focuses on regions exhibiting visible atrophy or abnormal morphology. In contrast, attribution maps for non-demented subjects appear more diffuse and less concentrated, reflecting the model's detection of normal anatomical patterns.

The results in Figure~\ref{fig:oasis_attributions} demonstrate that PWIG can effectively produce interpretable and clinically meaningful explanations for model predictions in medical imaging tasks.

\begin{figure*}[t]
    \centering
    \begin{subfigure}{0.19\textwidth}
        \includegraphics[width=\linewidth]{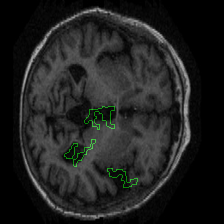}
        \caption{Mild 1}
    \end{subfigure}
    \begin{subfigure}{0.19\textwidth}
        \includegraphics[width=\linewidth]{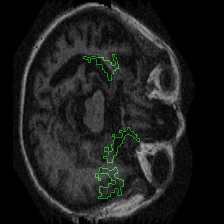}
        \caption{Mild 2}
    \end{subfigure}
    \begin{subfigure}{0.19\textwidth}
        \includegraphics[width=\linewidth]{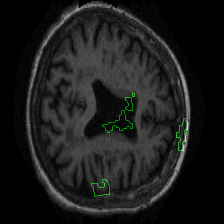}
        \caption{Mild 3}
    \end{subfigure}
    \begin{subfigure}{0.19\textwidth}
        \includegraphics[width=\linewidth]{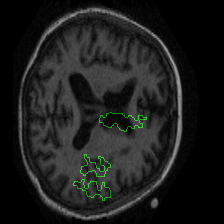}
        \caption{Moderate 1}
    \end{subfigure}
    \begin{subfigure}{0.19\textwidth}
        \includegraphics[width=\linewidth]{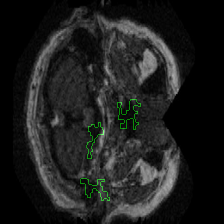}
        \caption{Moderate 2}
    \end{subfigure}
    
    \vspace{0.5em}
    
    \begin{subfigure}{0.19\textwidth}
        \includegraphics[width=\linewidth]{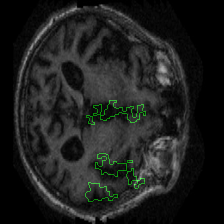}
        \caption{Moderate 3}
    \end{subfigure}
    \begin{subfigure}{0.19\textwidth}
        \includegraphics[width=\linewidth]{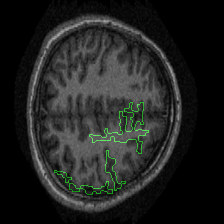}
        \caption{Non-demented 1}
    \end{subfigure}
    \begin{subfigure}{0.19\textwidth}
        \includegraphics[width=\linewidth]{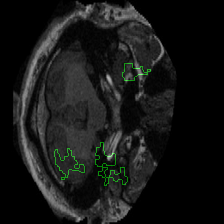}
        \caption{Non-demented 2}
    \end{subfigure}
    \begin{subfigure}{0.19\textwidth}
        \includegraphics[width=\linewidth]{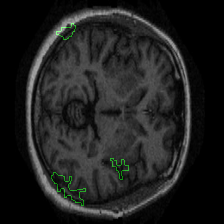}
        \caption{Very mild 1}
    \end{subfigure}
    \begin{subfigure}{0.19\textwidth}
        \includegraphics[width=\linewidth]{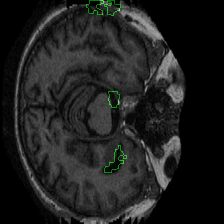}
        \caption{Very mild 2}
    \end{subfigure}
    
    \caption{PWIG attribution maps across dementia categories using the OASIS-1 dataset. Green overlays indicate regions with the highest attribution scores.}
    \label{fig:oasis_attributions}
\end{figure*}

\section{Conclusion}~\label{sec:conclusion}

In this work, we presented Path-Weighted Integrated Gradients (PWIG), a novel extension of the Integrated Gradients framework that introduces a weighting function to modulate the contribution of gradients along the interpolation path. PWIG retains key theoretical properties such as sensitivity (b) and implementation invariance, while enabling enhanced interpretability through dynamic gradient weighting. This flexibility allows users to highlight specific regions along the baseline-to-input trajectory and uncover subtle patterns of model reasoning.

The effectiveness of PWIG was valiadted through a case study in dementia classification using the OASIS-1 MRI dataset. It was shown that PWIG attribution maps successfully identified brain regions that align with clinical expectations across different dementia stages. The results demonstrate the practical value of PWIG in medical imaging applications where transparency and reliability are essential. Future work may explore the integration of manifold-constrained paths, adaptive weighting strategies, and applications to other domains such as time series and natural language processing. PWIG represents a promising step toward more expressive and informative explanations in deep learning.

\section*{Appendix}

\begin{theorem}[Implementation Invariance]
Let $F_1, F_2: \mathbb{R}^n \rightarrow \mathbb{R}$ be two network functions that are functionally equivalent, i.e., $F_1(x) = F_2(x)$ for all $x \in \mathbb{R}^n$. The Path-Weighted Integrated Gradients (PWIG) attribution for $F_1$ is identical to that for $F_2$.
\end{theorem}

\begin{proof}
The PWIG attribution for the $k^{th}$ feature of a function $F$ is defined as:
$$\text{PWIG}_k(x) := (x_k - x'_k) \int_{\alpha=0}^{1} g(\alpha) \frac{\partial F(x' + \alpha(x - x'))}{\partial x_k} d\alpha$$
An attribution method satisfies implementation invariance if its output depends only on the mathematical function being computed and not its specific implementation. Since $F_1(x) = F_2(x)$ for all $x$, their partial derivatives are also identical, i.e., $\frac{\partial F_1(x)}{\partial x_k} = \frac{\partial F_2(x)}{\partial x_k}$ for all $k \in \{1, ..., n\}$. The PWIG formula depends only on the input $x$, the baseline $x'$, the path between them, the weight function $g(\alpha)$, and the partial derivatives of the function $F$. As all these components are identical for $F_1$ and $F_2$, their PWIG attributions must be identical.
\end{proof}

\begin{theorem}[Linearity]
Let $F_1, F_2: \mathbb{R}^n \rightarrow \mathbb{R}$ be two network functions. Let $F = aF_1 + bF_2$ for some scalars $a, b \in \mathbb{R}$. Then the PWIG attribution for $F$ is the weighted sum of the attributions for $F_1$ and $F_2$.
\end{theorem}

\begin{proof}
The linearity axiom requires that $\text{PWIG}_k(aF_1 + bF_2) = a \cdot \text{PWIG}_k(F_1) + b \cdot \text{PWIG}_k(F_2)$ for any feature $k$. 
The partial derivative is a linear operator, so:
$$\frac{\partial F}{\partial x_k} = \frac{\partial}{\partial x_k}(aF_1 + bF_2) = a\frac{\partial F_1}{\partial x_k} + b\frac{\partial F_2}{\partial x_k}$$
Substituting this into the PWIG equation:
$$\text{PWIG}_k(F) = (x_k - x'_k) \int_{\alpha=0}^{1} g(\alpha) \left( a\frac{\partial F_1}{\partial x_k} + b\frac{\partial F_2}{\partial x_k} \right) d\alpha$$
Due to the linearity of the integral:
\begin{align*}
    \text{PWIG}_k(F) = a \left( (x_k - x'_k) \int_0^1 g(\alpha)\frac{\partial F_1}{\partial x_k}d\alpha \right) \\
+ b \left( (x_k - x'_k) \int_0^1 g(\alpha)\frac{\partial F_2}{\partial x_k}d\alpha \right)
\end{align*}
This simplifies to:
$$\text{PWIG}_k(F) = a \cdot \text{PWIG}_k(F_1) + b \cdot \text{PWIG}_k(F_2)$$
\end{proof}

\begin{theorem}[Dummy Axiom/Sensitivity (b)]
Let $F: \mathbb{R}^n \rightarrow \mathbb{R}$ be a function that does not mathematically depend on a variable $x_k$ for some $k \in \{1, ..., n\}$. Then the PWIG attribution for $x_k$ is always zero. 
\end{theorem}

\begin{proof}
The dummy axiom states that if a function does not depend on a variable, the attribution to that variable must be zero. If $F$ does not depend on $x_k$, its partial derivative with respect to $x_k$ is identically zero for all inputs:
$$\frac{\partial F(x)}{\partial x_k} = 0 \quad \forall x \in \mathbb{R}^n$$
The PWIG attribution for $x_k$ is:
$$\text{PWIG}_k(x) = (x_k - x'_k) \int_{\alpha=0}^{1} g(\alpha) \frac{\partial F(x' + \alpha(x - x'))}{\partial x_k} d\alpha$$
Since the term $\frac{\partial F}{\partial x_k}$ is zero everywhere along the integration path, the integral evaluates to zero. Therefore:
$$\text{PWIG}_k(x) = (x_k - x'_k) \times 0 = 0$$
\end{proof}

\begin{theorem}[Symmetry-Preserving]
Let variables $x_i$ and $x_j$ be symmetric with respect to a function $F$. For any input $x$ where $x_i = x_j$ and any baseline $x'$ where $x'_i = x'_j$, the PWIG attributions for $x_i$ and $x_j$ are identical, i.e., $\text{PWIG}_i(x) = \text{PWIG}_j(x)$.
\end{theorem}

\begin{proof}
An attribution method is symmetry-preserving if symmetric variables receive identical attributions when their values are identical in both the input and the baseline.
Let the path from $x'$ to $x$ be $\gamma(\alpha) = x' + \alpha(x - x')$. Given that $x_i = x_j$ and $x'_i = x'_j$, it follows that the path components are equal for all $\alpha \in [0, 1]$:
$$\gamma_i(\alpha) = x'_i + \alpha(x_i - x'_i) = x'_j + \alpha(x_j - x'_j) = \gamma_j(\alpha)$$
Due to the symmetry of the function $F$ with respect to variables $i$ and $j$, their partial derivatives are equal at any point where the values for those variables are equal. Since $\gamma_i(\alpha) = \gamma_j(\alpha)$ along the entire path, we have:
$$\frac{\partial F(\gamma(\alpha))}{\partial x_i} = \frac{\partial F(\gamma(\alpha))}{\partial x_j}$$
The PWIG attributions for $x_i$ and $x_j$ are:
$$\text{PWIG}_i(x) = (x_i - x'_i) \int_0^1 g(\alpha) \frac{\partial F(\gamma(\alpha))}{\partial x_i} d\alpha$$
$$\text{PWIG}_j(x) = (x_j - x'_j) \int_0^1 g(\alpha) \frac{\partial F(\gamma(\alpha))}{\partial x_j} d\alpha$$
Since $(x_i - x'_i) = (x_j - x'_j)$ and the integrands are identical, it follows that $\text{PWIG}_i(x) = \text{PWIG}_j(x)$.
\end{proof}

\end{document}